\documentclass[twoside]{article}

\usepackage{aistats2020}
\usepackage{amsthm}
\usepackage{amssymb}
\usepackage{amsfonts}
\usepackage[utf8]{inputenc} 
\usepackage[T1]{fontenc}    
\usepackage{hyperref}       
\usepackage{url}            
\usepackage{booktabs}       
\usepackage{amsfonts}       
\usepackage{nicefrac}       
\usepackage{microtype}      
\usepackage{algpseudocode,algorithm,algorithmicx} 
\usepackage{dsfont}
\usepackage{graphicx}
\usepackage{xcolor}

\newcommand{\X}{\mathbf{X}}
\newcommand{\Y}{\mathbf{Y}}
\newcommand{\x}{\mathbf{x}}
\newcommand{\y}{\mathbf{y}}

\newcommand{\bP}{\boldsymbol{\Pi}}
\newcommand{\bBeta}{\boldsymbol{\beta}}

\newcommand{\R}{\mathbb{R}}
\newcommand{\E}{\text{E}}
\newcommand{\Var}{\text{Var}}

\newtheorem{proposition}{Proposition}
\newtheorem{assumption}{Assumption}

\begin{document}

%

%

\twocolumn[

\aistatstitle{Robust approximate linear regression without correspondence }

\aistatsauthor{
   Amin Nejatbakhsh \\
  Department of Neurobiology and Behavior\\
  Columbia University\\
  \And
  Erdem Varol \\
  Department of Statistics\\
  Columbia University
  \vspace{2ex}
}

]

\begin{abstract}
We propose methods for estimating correspondence between two point sets under the presence of outliers in both the source and target sets. The proposed algorithms expand upon the theory of the regression without correspondence problem to estimate transformation coefficients using unordered multisets of covariates and responses. Previous theoretical analysis of the problem has been done in a setting where the responses are a complete permutation of the regressed covariates. This paper expands the problem setting by analyzing the cases where only a subset of the responses is a permutation of the regressed covariates in addition to some covariates being outliers. We term this problem \textit{robust regression without correspondence} and provide several algorithms based on random sample consensus for exact and approximate recovery in a noiseless and noisy one-dimensional setting as well as an approximation algorithm for multiple dimensions. The theoretical guarantees of the algorithms are verified in simulated data. We demonstrate an important computational neuroscience application of the proposed framework by demonstrating its effectiveness in a \textit{Caenorhabditis elegans} neuron matching problem where the presence of outliers in both the source and target nematodes is a natural tendency.
\end{abstract}

\section{Introduction}

Point set registration is one of the central problems in computer vision that involves the optimization of a transformation that aligns two sets of point clouds~\cite{van2011survey,tam2013registration}. Point set registration have been applied in numerous fields including but not limited to robotics~\cite{zhang2015visual}, medical imaging~\cite{audette2000algorithmic}, object recognition~\cite{drost2010model}, panorama stitching~\cite{bazin2014globally} and computational neuroscience~\cite{bubnis2019probabilistic}.
 The types of allowable transformations and energy functions utilized in the cost function have differentiated varying methods \cite{besl1992method,myronenko2010point,zhou2016fast,hast2013optimal,irani1999combinatorial,aiger20084,mount1999efficient,tam2013registration,indyk1999geometric,pokrass2013sparse}. In general, point set registration methods employ an iterative strategy of solving the transformation and updating the matching which works well in practice but there are no guarantees for reaching the global optima ~\cite{chetverikov2002trimmed}. Only a few methods have provided approximate globally optimal solutions \cite{yang2016go,zhou2016fast}. These methods rely on severe constraints of the transformation domains, such as the 3D rotation group SO(3), in order to employ branch and bound techniques on discretizations. 

Theoretical analysis of the recovery guarantees of point set registration has not been performed for a general number of dimensions until recently when it was termed as \textit{unlabelled sensing} by \cite{unnikrishnan2015unlabeled} as a problem with duality connections with the well-known problem of compressed sensing~\cite{donoho2006compressed}. In this problem, similar to linear regression, the response signal is modeled as a linear combination of a set of covariates. However, the correspondence of the responses to the covariates is modeled as having been shuffled by an unknown permutation matrix. For this reason, the problem has also been termed as \textit{linear regression with shuffled labels} \cite{abid2017linear}, \textit{linear regression with an unknown permutation} \cite{pananjady2016linear}, \textit{homomorphic sensing} \cite{tsakiris2019homomorphic} or \textit{linear regression without correspondence} (RWOC) \cite{hsu2017linear}, the latter of which will be used to refer to the problem herein.  Although RWOC is, in general, an NP-hard problem \cite{pananjady2016linear}, there have been several advances in recent years to propose signal to noise ratio (SNR) bounds for recovery of the permutation matrix and the regression coefficients \cite{pananjady2016linear,unnikrishnan2018unlabeled}. Conversely, the same works have also analyzed the SNR and sampling regime by which no recovery is possible.

Nevertheless, the computer vision community has attempted to solve the point set registration problem through consideration of outliers and missing correspondences, which are typically encountered in real-world applications. A common technique used in point set registration to robustify the optimization against outliers is to employ random sampling consensus (RANSAC) subroutines  \cite{fischler1981random,torr2000mlesac,yang2019polynomial}. The main advantages of RANSAC are that the randomization procedure employed can severely reduce the computational cost of an otherwise combinatorial search.

Motivated by applications in computational neuroscience such as matching the neuronal populations of \textit{Caenorhabditis elegans} (\textit{C. elegans}) across different nematodes, we aim to unify the ideas presented in RWOC literature and robust point set registration methods to provide provably approximate solutions to the RWOC problem in the presence of outliers and missing measurements commonly encountered in fluorescence microscopy data. Robustly and automatically matching and identifying neurons in \textit{C. elegans} could expedite the post-experimental data analysis and hypothesis testing cycle \cite{bubnis2019probabilistic,kainmueller2014active,nguyen2017automatically,yemini2019neuropal}.

\subsection{Main contributions}
The main contributions presented in this paper are the introduction of randomized algorithms for the recovery of the regression coefficients in the RWOC problem that takes into account noise, missing data, and outliers. Hsu et al.~\cite{hsu2017linear} provide algorithms for the noisy case without generative assumptions; their algorithm takes into account square permutation matrices, which assumes that the entire signal is captured in the responses and does not take into account any missing correspondences or outliers. Unnikrishnan et al.~\cite{unnikrishnan2015unlabeled,unnikrishnan2018unlabeled} provide combinatorial existence arguments. Tsakiris et al.~\cite{tsakiris2019homomorphic} provide an algorithm that takes into account missing correspondences or outliers but not both. Our method is designed for the practical purpose of matching point clouds that may have noisy measurements, missing correspondences, and outliers. This is undoubtedly the case in the application domain of neuron tracking and matching in biological applications. Specifically, we demonstrate the efficacy of the proposed method in the identification and tracking of in-vivo (\textit{C. elegans}) neurons. 
In summary, our contributions are four-fold:
\begin{enumerate}
    \item We introduce the notion of \textit{"robust" regression without correspondence} (rRWOC) that models missing correspondences between responses and covariates as well as completely missed associations in the form of outliers and missing data.
    \item We introduce a polynomial-time algorithm to find the exact solution for the one-dimensional noiseless rRWOC and the approximate solution in the noisy regime.
    \item We introduce a randomized approximately correct algorithm that is more efficient than pure-brute force approaches in multiple dimensional rRWOC.
    \item We demonstrate the computational neuroscience application of our approach to point-set registration problems in the context of automatically matching and identification of the cellular layout of the nervous system of the nematode \textit{C. elegans}.
\end{enumerate}
\begin{figure*}[!ht]
    \centering   \includegraphics[height=4cm]{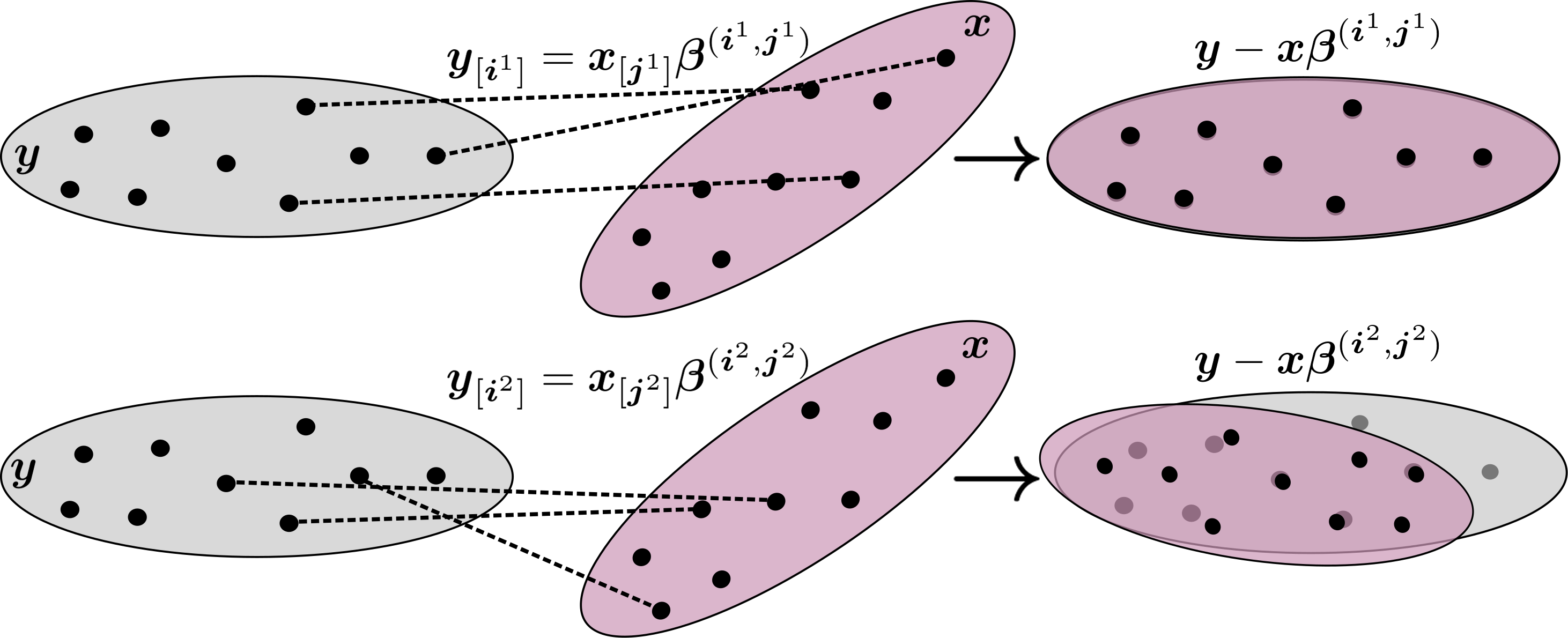}
    \caption{Geometric intuition of the proposed algorithms. \textbf{Top:} Solving the linear regression problem on the correct d-tuple correspondence yields regression coefficients that result in a transformed set with maximal alignment. \textbf{Bottom:} Incorrect correspondences of the sampled $d$-tuples results in sub-maximal alignment.}
    \label{schematic}
\end{figure*}

\subsection{Paper organization}
In section~\ref{model}, we introduce our statistical regression model (rRWOC) that accounts for permuted correspondences, outliers, and noise. We then demonstrate the added computational complexity of recovery of rRWOC in contrast with simple linear regression and RWOC in a one-dimensional case in section~\ref{one_d}. In section~\ref{multi_d}, we provide a randomized algorithm for the rRWOC problem in multiple dimensions with convergence bounds. Lastly, in section~\ref{sim_results}, we verify the theoretical recovery guarantees in simulated experiments and in section~\ref{c_elegans} show the neuroscience application of the proposed algorithms in the \textit{C. elegans} neuron matching problem.
\section{Regression model}\label{model}
First, we introduce notation. Let $\X = [\x_1|\x_2|\ldots|\x_m]^T\in\R^{m\times d}$ and $\Y = [\y_1|\y_2|\ldots|\y_n]^T\in\R^{n\times d}$ denote two d-dimensional point sets consisting of $m$ and $n$ points, respectively. Let us call $\X$ the reference or source set. Let $\Y$ denote the target set which may contain outliers and missing correspondences. Note that the points in $\X$ that are missing correspondences in $\Y$ can be seen as outliers in the source set, hence justifying our claim that we model outliers in both the source and target sets. 

Let the set of indices $\mathcal{I} = \lbrace i_1,\ldots,i_{|\mathcal{I}|}\rbrace \subseteq [n]$ denote the indices of $\y_j$ which are inliers. Conversely, let $\mathcal{O} = \lbrace o_1,\ldots,o_{|\mathcal{O}|}\rbrace \subseteq [n]$ denote set of indices of $\y_j$ which are outliers. By construction, these sets are a disjoint partition of the entire index set of target points: $\mathcal{I} \bigcup \mathcal{O} = [n]$ and $\mathcal{I} \bigcap \mathcal{O} = \emptyset$. Let $\bP \in \mathcal{P}^{n\times m}$ denote a possibly unbalanced permutation matrix where there are at most $\min\lbrace n,m\rbrace$ ones placed such that no row or column has more than a single one. All other entries are zeroes. Let $\pi(i)$ denote the location of the one in the $i$th row of the permutation matrix $\bP$. Next, let $\bBeta \in \mathbb{R}^{d\times d}$ denote the regression coefficients and $\epsilon \sim \mathcal{N}(0,\nu 
\boldsymbol{I})$ denote zero-mean Gaussian noise. Lastly, let $\text{U}[\mathcal{C}]$ denote the uniform distribution within some closed convex set $\mathcal{C}$. Given these definitions, we can define the \textbf{robust regression without correspondence} (rRWOC) model as
\begin{align}
\y_{i_j} &= \x_{\pi(i_j)}\bBeta + \epsilon~\quad~&\text{for}~i_j \in \mathcal{I} \nonumber \\
\y_{o_l} &\sim \text{U}[\mathcal{C}]~\quad~&\text{for}~o_l \in \mathcal{O}
\end{align}

Note that the bias terms in the regression can be modeled by padding $\x$ with a constant column of ones.

In contrast with linear regression, where the sole objective is to recover the coefficients $\bBeta$, the two-fold objective of RWOC is to recover the correct permutation matrix $\bP$, and the regression coefficients $\bBeta$. To add to the complexity of the problem, the three-fold objective of rRWOC is to recover the inlier set $\mathcal{I}$, the permutation $\bP$, and the coefficients $\bBeta$.

\section{Algorithms}
To aid in the recovery of the solution in rRWOC, we introduce the following assumption.

\begin{assumption}[Maximal inlier set]\label{assumption_1}
For point sets $\X$, $\Y$, there exists a triple $\lbrace \mathcal{I}^*,\bBeta^*,\bP^* \rbrace$ that is maximal in the sense that $n\geq |\mathcal{I}^*| \geq |\mathcal{I}'|$  such that any other triple $\lbrace \mathcal{I}',\bBeta',\bP' \rbrace$ is not considered to be the underlying regression model.
\end{assumption}

Assumption~\ref{assumption_1} allows the identifiability of whether a given hypothetical index set can be considered to be the true underlying inlier set or not. In practical terms, suppose we generate simulated data with $n$ points in $\Y$ of which $k > n/2$ are outliers generated uniformly and the remainder generated with respect to a coefficient $\bBeta^{\mathcal{I}}$ such that $\Y_{[\mathcal{I}]} = \X_{\pi(\mathcal{I})}\bBeta^{\mathcal{I}} + \epsilon^{\mathcal{I}}$. There may be cases such that uniformly generated "outliers", $\Y_{[\mathcal{O}]}$, are structured such that there exists a coefficient $\bBeta^{\mathcal{O}}$ and permutation $\bP^{\mathcal{O}}$ such that $ \Y_{[\mathcal{O}]} = \X_{\pi(\mathcal{O})}\bBeta^{\mathcal{O}} + \epsilon^{\mathcal{O}}$ where $\Var(\epsilon^{\mathcal{I}}) \geq \Var(\epsilon^{\mathcal{O}})$. In this case, $\bBeta^{\mathcal{O}}$ is identifiable but not verifiable as "correct."

Equipped with the rRWOC model and assumption~\ref{assumption_1}, we now demonstrate the progressive increase in the complexity of recovery of ordinary linear regression, RWOC, and rRWOC in one-dimension.

\subsection{Optimal regression in $d=1$}\label{one_d}
Linear regression in one-dimension with known correspondences, no offset term and no outliers can be obtained in $O(n)$ time using the univariate normal equation: $\beta_{OLS} = \frac{\sum_i^n y_i x_{\pi(i)}}{\sum_i^n x_{\pi{i}}^2}.$ On the other hand, RWOC in the one-dimensional case with \textit{no noise} can be solved in $O(n \log(n))$ steps via the method of moments and a simple sorting operation. Namely, first, the regressor $\beta_{RWOC}$ can be estimated using the ratio of the first moments of the covariates to the responses:
\begin{align}
    \beta_{RWOC} = \frac{\sum_{i=1}^n y_i}{ \sum_{i=1}^n x_i}
\end{align}
and then the permutation can be recovered using the re-arrangement inequality~\cite{beckenbach2012inequalities},
\begin{align}
&\min_{\bP} \sum_{i=1}^n (y_i - \hat{y}_{\pi(i)})^2 = \sum_{i=1}^n (y_{(i)} - \hat{y}_{(i)})^2 =\\
&\|\bP_y\y - \bP_{\hat{y}}\hat{\y}\|_2^2 \longrightarrow \bP_{RWOC} = \bP_y^T\bP_{\hat{y}}\nonumber
\end{align}
where $y_{(i)}$ denotes sorted $y_i$ and $\hat{y}_{(i)}$ denotes sorted $x_i\beta_{RWOC}$ and $\bP_y$ and $\bP_{\hat{y}}$ denote the permutation matrices that capture the sorting operations. 

In the case with outlier elements in $\y$, the problem is non-trivial, even in one dimension, since sorting does not allow the identification of outliers. To solve the one dimensional rRWOC, we introduce algorithm~\ref{alg:1d_exhaustive} which recovers the triplet $\lbrace \mathcal{I}^*,\bBeta^*,\bP^* \rbrace$ in an exhaustive fashion.

\begin{algorithm}[!htb]
\caption{One dimensional robust regression without correspondence - Exhaustive approach}\label{alg:1d_exhaustive}
\begin{algorithmic}
\State{\textbf{Input: }} Reference set: $\lbrace x_1,\ldots,x_m \rbrace$, target set: $\lbrace y_1,\ldots,y_n \rbrace$, outlier margin: $\nu$
\State{\textbf{Require: }} $k < \frac{n}{2}$ (number of outliers)
\For{$i=1,\ldots,n$}
\For{$j=1,\ldots,m$}
\State Compute $\beta^{i,j} = y_i/x_j$
\State Compute linear assignment~\cite{kuhn1955hungarian}:
\State $\bP^{i,j}\leftarrow\underset{\bP \in \mathcal{P}^{n \times m}}{\arg\min} \| \x\beta^{i,j} - \bP^T\y \|_2^2$
\State Compute hypothetical inliers:
\State $\mathcal{I}^{i,j} = \lbrace l : |x_{\pi^{i,j}(l)}\beta^{i,j} - y_l| \leq \nu \rbrace$
\EndFor
\EndFor\\
\Return $(i^*,j^*) = \underset{(i,j)}{\arg\max}|\mathcal{I}^{i,j}|$
, $\mathcal{I}^* = \mathcal{I}^{i^*,j^*}$,
\indent \indent $\bP^* = \bP^{i^*,j^*}$,~$\beta^* \leftarrow \frac{\sum_{l \in \mathcal{I}^*} y_{l}x_{\pi^*(l)}}{\sum_{l \in \mathcal{I}^*} x_{\pi^*(l)^2}}$ 
\end{algorithmic}
\end{algorithm}

\begin{proposition}[Correctness of Algorithm~\ref{alg:1d_exhaustive}] Suppose there exist $n-k$ inliers in $\y$ and that $k<n/2$. Then algorithm~\ref{alg:1d_exhaustive} yields the correct regression coefficient $\beta^* = \beta$ with probability 1 for noiseless data and with high probability for noisy data with an appropriately selected margin parameter $\nu$.
\end{proposition}
\begin{proof}(The full proof is included in supplementary material) The overview of the proof is as follows. In the noiseless case, if $j=\pi(i)$ then $\beta^{i,j} = \frac{y_i}{x_j} = \beta^*$. The projection $\x\beta^{i,j}$ maps all reference points to their exact corresponding reference points. Thus the Hungarian algorithm will yield these as the assignments since they incur minimal cost. Therefore, we will have $|\mathcal{I}^{i,j}| \geq n-k$. The cardinality of inliers is lower bounded and not equal to $n-k$ since outlier points may by chance be transformed to points in $\y$ as well. Contrarily, suppose the transformation $\beta^{i,l}$ for $l \neq \pi(i)$ yields a larger hypothesized inlier set $\mathcal{I}^{i,l}$, such that  $|\mathcal{I}^{i,l}| > |\mathcal{I}^{i,j}|$ then this means that there are more points in  $\x\beta^{i,l}$ that are closer to $\y$ than $\x\beta^{i,j}$, contradicting the assumption that $n-k$ is the maximal inlier set.
\end{proof}
The time complexity of algorithm~\ref{alg:1d_exhaustive} can be analyzed as follows. The main computational cost is due to linear assignment which incurs a cost of $O(\max\lbrace m,n \rbrace^3)$ if \cite{jonker1986improving} variant is used. Linear assignment is repeated $mn$ times. If $m$ and $n$ are of the same order, then algorithm~\ref{alg:1d_exhaustive} has complexity $O(n^5)$.

However, if the ratio of inliers to outliers is relatively high, then it is possible to use randomization procedures like RANSAC~\cite{fischler1981random,torr2000mlesac} to speed up the algorithm to yield the correct regression coefficient with high probability. This is demonstrated in algorithm~\ref{alg:1d_randomized}.

\begin{proposition}[Correctness of Algorithm~\ref{alg:1d_randomized}] Suppose there are $n-k$ inliers in $\x$ and that $k<n/2$. In $q \geq \frac{\log(1-\delta)}{\log(1-\frac{n-k}{mn})}$ iterations, algorithm~\ref{alg:1d_randomized} yields the correct regression coefficient $\beta^* = \beta$ with probability $\delta\in(0,1)$ for an appropriately selected margin parameter $\nu$.
\end{proposition}
\begin{proof}
The success of algorithm~\ref{alg:1d_exhaustive} relies on the fact that the exhaustive search eventually hits a tuple $(i,j)$ such that $j = \pi(i)$ which yields the correct regression coefficient. Therefore, when randomly sampling $(i,j) \sim [n]\times[m]$, the probability of choosing a corresponding pair is $\frac{n-k}{n} \frac{1}{m}$. The probability of iterating $q$ times such hat no correct correspondence is selected is $(1-(n-k)/(nm))^q = (1-\delta)$ where $\delta$ is the desired success rate. Taking logs yields, $q = \frac{\log(1-\delta)}{\log(1-(n-k)/(nm))}$
\end{proof}
\begin{algorithm}[!htb]
\caption{One dimensional robust regression without correspondence - Randomized approach}\label{alg:1d_randomized}
\begin{algorithmic}
\State{\textbf{Input: }} Reference set: $\lbrace x_1,\ldots,x_m \rbrace$, target set: $\lbrace y_1,\ldots,y_n \rbrace$, $\delta$ (probability of success), outlier margin: $\nu$
\State{\textbf{Require: }} $k < \frac{n}{2}$ (number of outliers)
\For{$t=1,\ldots,q$}
\State Sample $i \sim [n]$ and sample $j\sim[m]$
\State Compute $\beta^t = y_i/x_j$
\State Compute linear assignment~\cite{kuhn1955hungarian}: 
\State $\bP^t\leftarrow\underset{\bP \in \mathcal{P}^{n \times m}}{\arg\min} \| \x\beta^{t} - \bP^T\y \|_2^2$
\State Compute hypothetical inliers:
\State $\mathcal{I}^t = \lbrace l : |x_{\pi^t(l)}\beta^t - y_l| \leq \nu \rbrace$
\EndFor\\
\Return $t^* = \underset{t}{\arg\max}|\mathcal{I}^t|$
, $\mathcal{I}^* = \mathcal{I}^{t^*}$,\\
\indent \indent $\bP^* = \bP^{i^*,j^*}$,~$\beta^* \leftarrow \frac{\sum_{l \in \mathcal{I}^*} y_{l}x_{\pi^*(l)}}{\sum_{l \in \mathcal{I}^*} x_{\pi^*(l)^2}}$ 
\end{algorithmic}
\end{algorithm}

The time complexity of randomized algorithm~\ref{alg:1d_randomized} is $O\bigg(\frac{\log(1-\delta)}{\log(1-(n-k)/n^2)}n^3\bigg)$. 
\begin{figure*}[!tb]
    \centering \includegraphics[width=0.32\linewidth,trim=0 0.1cm 1.1cm 0,clip]{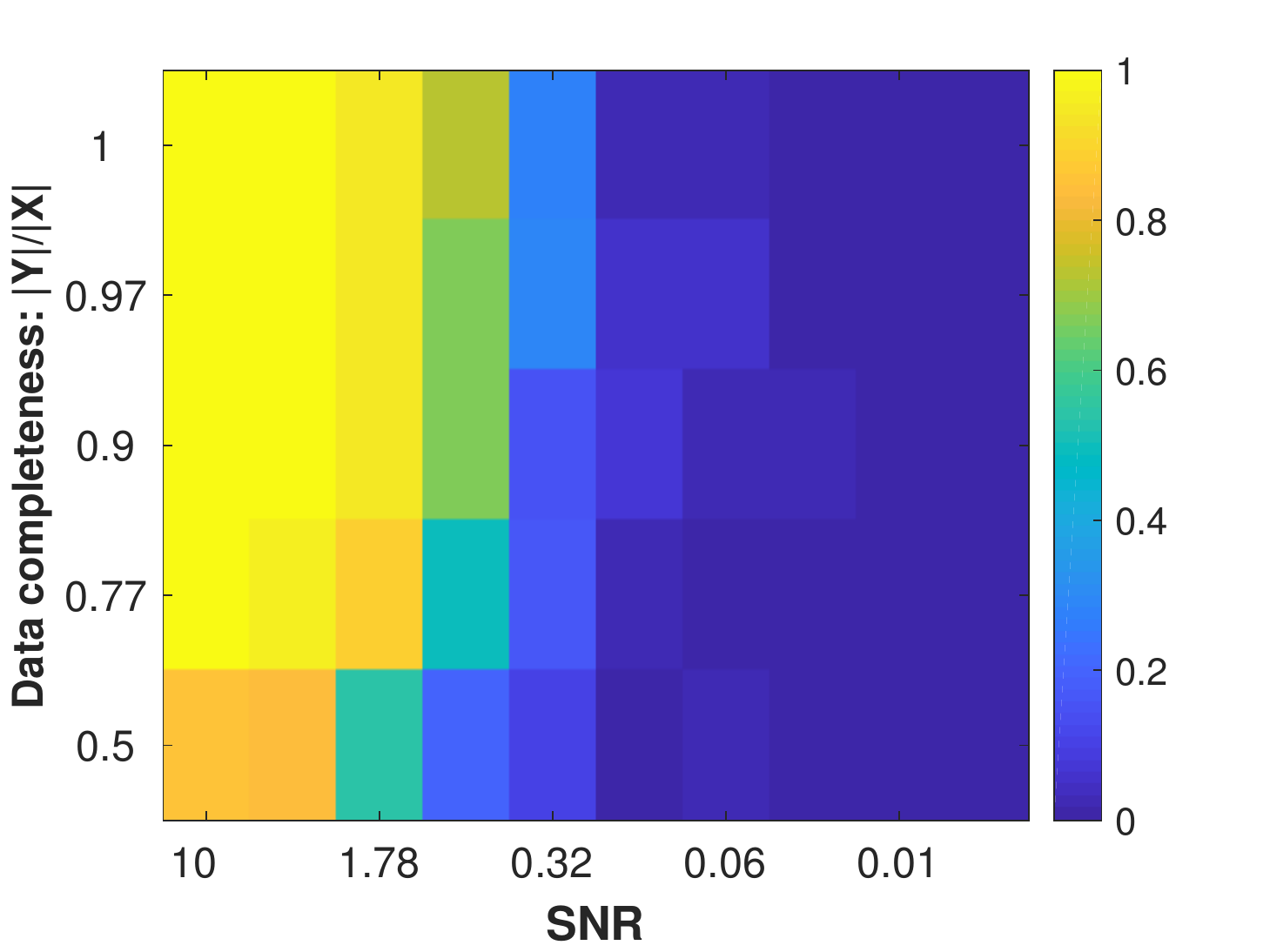}
        \includegraphics[width=0.32\linewidth,trim=0 0.1cm 1.1cm 0,clip]{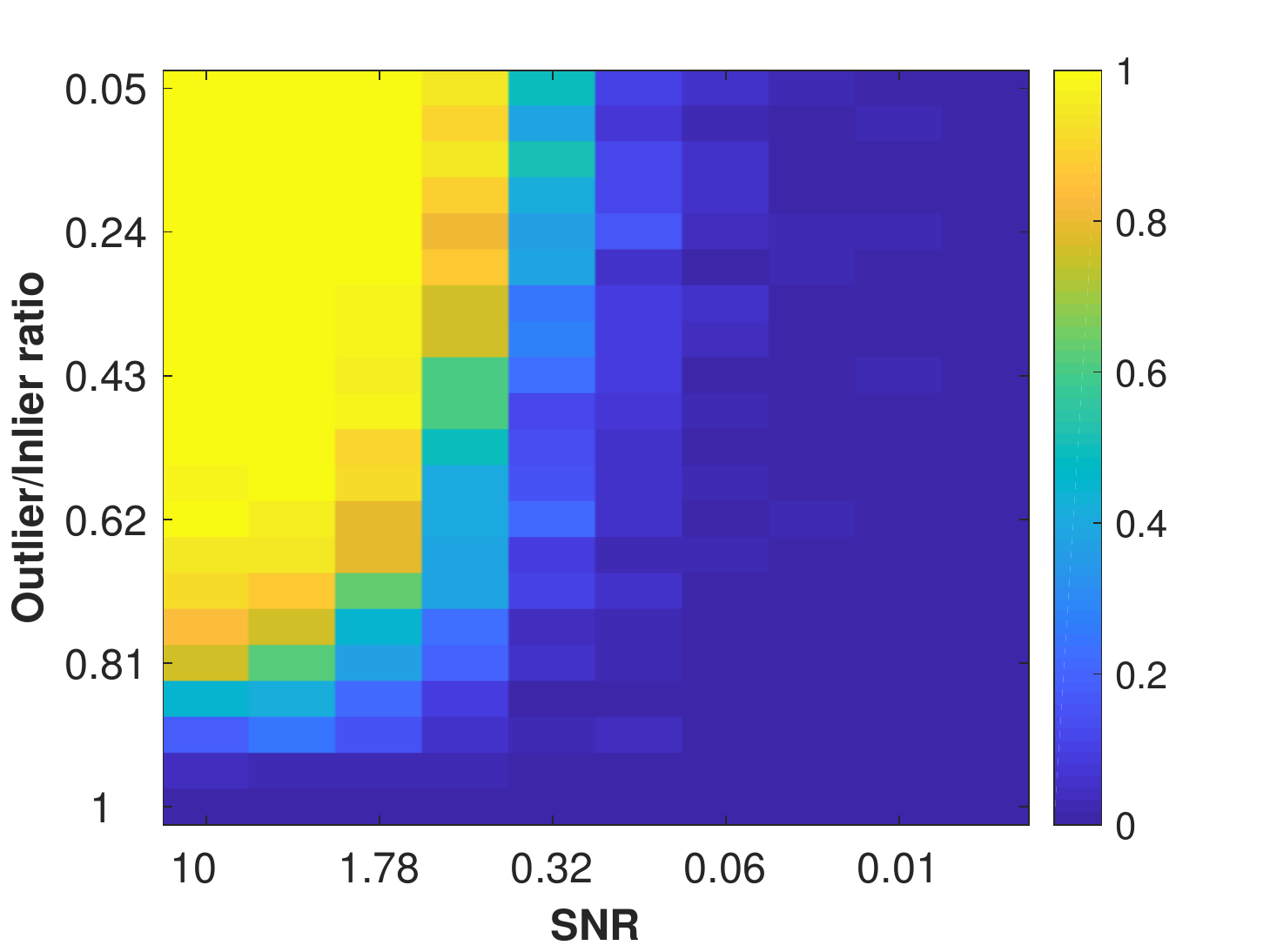}
  \includegraphics[width=0.32\linewidth,trim=0 0.1cm 1.1cm 0,clip]{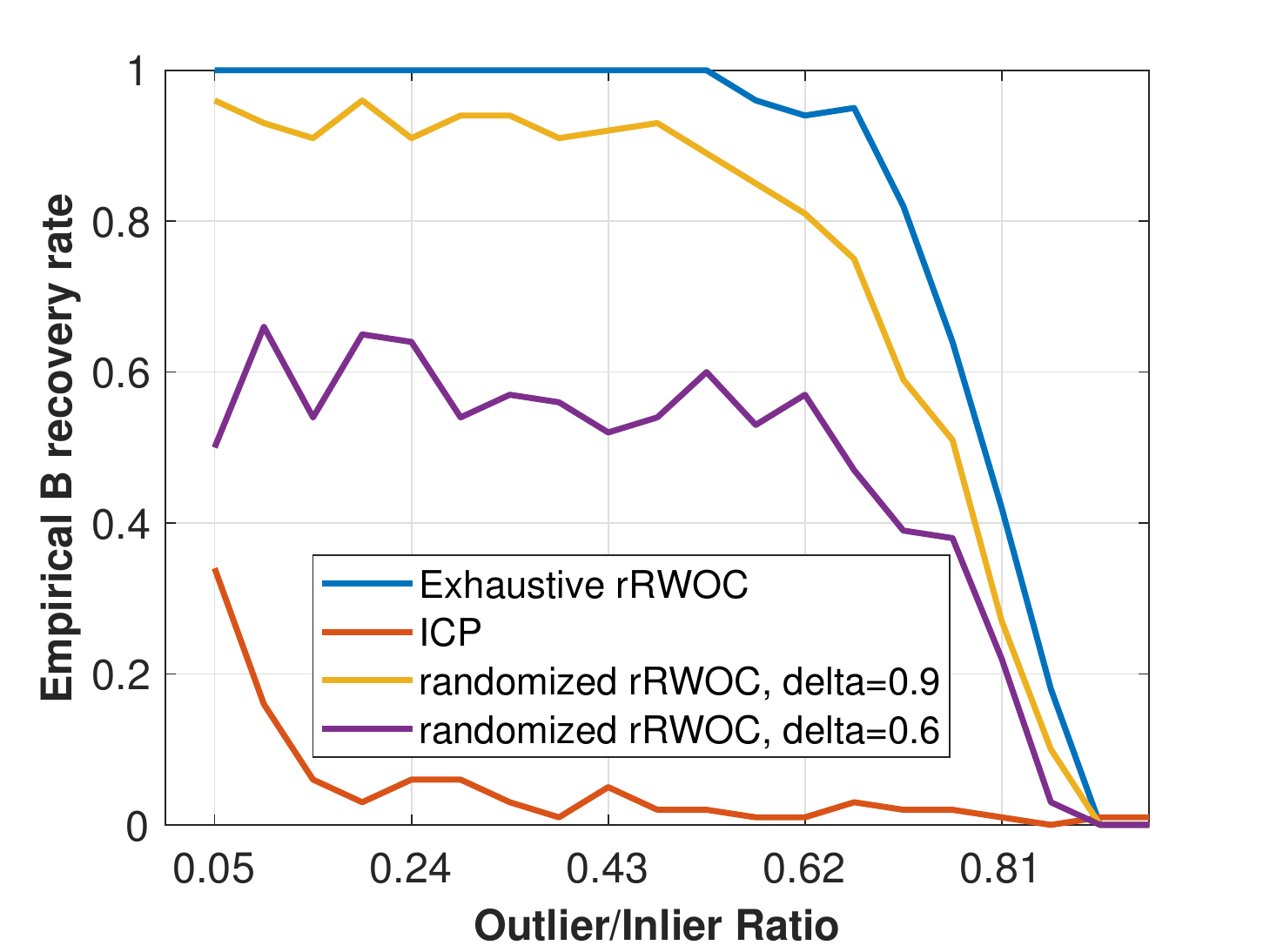}
    \caption{\textbf{Left:} Recovery rate (colorbar) vs. missing data ratio (y-axis) vs. SNR (x-axis), \textbf{Middle:} Recovery rate (colorbar) vs. outlier ratio (y-axis) vs. SNR (x-axis), \textbf{Right:} Recovery rate (y-axis) vs. outlier ratio (x-axis), blue: rRWOC, red: ICP, yellow: randomized rRWOC ($\delta=0.9)$), purple: randomized rRWOC ($\delta=0.6$)}
    \label{simulated_resultsl}
\end{figure*}

\subsection{Randomized approximation algorithm for $d\geq 2$}\label{multi_d}

The exhaustive approach for the $d\geq 2$ dimensional case requires $\binom{n}{d}\binom{m}{d}$ $d$-subset comparisons of $\X$,$\Y$ in order to guarantee hitting correct (in the noiseless case) or approximately correct (in the noisy case) regression coefficients, with complexity $O(m^d n^d)$. However, especially in higher dimensions, the randomized procedure enables a substantial reduction of iterations to yield a high probability correct triplet of inlier set, permutation, and regression coefficients. The randomized algorithm for rRWOC in $d\geq 2$ is demonstrated in algorithm~\ref{alg:ranpac}. Conceptually, the idea of the algorithm is illustrated in figure~\ref{schematic}. Random ordered $d$-tuples of reference and target point sets are sampled and are used to align the remainder of the point set. The number of hypothetical inliers for each hypothetical correspondence is assessed by checking whether the transformed reference points are arbitrarily close to a target point. With high probability, if correct a $d$-tuple correspondence is captured, the number of transformed reference points matching a target point will be high (Figure~\ref{schematic} top), otherwise it will result in a partial coverage (Figure~\ref{schematic} bottom).

\begin{algorithm}[!htb]
\caption{Robust regression without correspondence - Randomized approach}\label{alg:ranpac}
\begin{algorithmic}
\State \textbf{Input: } $\X = [ \x_1 | \ldots | \x_m ]^T \in\R^{m \times d}$ (reference points), $\Y=[ \y_1 | \ldots | \y_n ]^T\in\R^{n \times d}$ (target points), $\delta$ (probability of success), $\nu$ (outlier margin)
\State{\textbf{Require: }} $k < \frac{n}{2}$ (number of outliers)
  \State \textbf{Output: } $\mathcal{I} \subseteq [n]$ (index of inliers), $\hat{\bP} \in \mathcal{P}^{m \times |\mathcal{I}|}$ (permutation matrix), $\hat{\bBeta} \in \R^{d\times d}$ (regression coefficients)
  \medskip
\For{$t=1,\ldots,q$}
\State Sample $\boldsymbol{i} = (i_1,\ldots,i_d) \sim [n]^d$ w/o replacement
\State Sample $\boldsymbol{j} = (j_1,\ldots,j_d) \sim [m]^d$ w/o replacement
\State Compute $\bBeta^t = \underset{\bBeta}{\arg\min} \|\X_{[\boldsymbol{j}]}\bBeta - \Y_{[\boldsymbol{i}]} \|_F^2$
\State Compute linear assignment via~\cite{kuhn1955hungarian}:
\State $\bP^t\leftarrow\underset{\bP \in \mathcal{P}^{m\times n}}{\arg\min} \| \X\bBeta^t - \bP\Y \|_F^2$
\State Compute hypothetical inliers:
\State $\mathcal{I}^t = \lbrace l : \|\x_{\pi^t(l)}\bBeta^t - \y_l\|_2 \leq \nu \rbrace$
\EndFor\\
\Return $t^* = \arg\max_t  |\mathcal{I}^t|$, $\mathcal{I}^* = \mathcal{I}^{t^*}$,\\
\indent \indent$\bP^* = \bP^{t^*}_{\mathcal{I}^*}$,$\bBeta^* \leftarrow \arg\min_{\bBeta} \|\X_{\pi^*(\mathcal{I^*})} \bBeta - \Y_{\mathcal{I^*}}\|_F^2$
\end{algorithmic}
\end{algorithm}

\begin{proposition}
For $q\geq \frac{\log(1-\delta)}{\log\bigg(1-\frac{\binom{m-k}{d}}{\binom{m}{d}\binom{n}{d}}\bigg)}$, algorithm~\ref{alg:ranpac} recovers $\bBeta^*$ and $\bP^*$ and the set of inliers for the noiseless case with probability $(1-\delta)$ using arbibrarily small $\nu$. For sufficiently small noise variance and appropriately chosen $\nu$, algorithm~\ref{alg:ranpac} recovers approximate $\bBeta^*$ with high probability.
\end{proposition}
\begin{proof}
Analogous to the analysis of algorithm~\ref{alg:1d_randomized}, the probability of drawing $d$ inliers out of $n$ points with k outliers in $\Y$ is $\frac{\binom{n-k}{d}}{\binom{n}{d}}$. The probability of matching the drawn inliers with the $d$ corresponding sampled reference points in $\X$ is $\frac{1}{\binom{m}{d}}$. Probability that any draw is not going to match is $1-\frac{\binom{n-k}{d}}{\binom{m}{d}\binom{n}{d}}$. The probability that $q$ draws will be incorrect is $\bigg(1-\frac{\binom{m-k}{d}}{\binom{m}{d}\binom{n}{d}}\bigg)^q$. If we set this to be the probability of failure $(1-\delta)$, we then have the estimate for the number of draws we need to make as $q(\delta,n,m,k)\geq \log(1-\delta)/\log\bigg(1-\frac{\binom{m-k}{d}}{\binom{m}{d}\binom{n}{d}}\bigg)$
\end{proof}

The complexity of algorithm~\ref{alg:ranpac} can be analyzed as follows. In each inner loop, the regression coefficient solution requires $O(d^3)$ time, the Hungarian algorithm requires $O(nmd)$ to compute the input distance matrix and then $O(\max\lbrace n,m \rbrace^3)$ to optimize the permutation matrix. The rest of the operations are $O(d)$. Therefore, the overall time complexity is
\small
\begin{align}
    O\bigg(\frac{\log(1-\delta)}{\log\bigg(1-\frac{\binom{m-k}{d}}{\binom{m}{d}\binom{n}{d}}\bigg)}(d^3 + nmd + \max\lbrace n,m \rbrace^3)\bigg).
\end{align}
\normalsize

\begin{figure*}[!htb]
    \centering
    \includegraphics[trim=0cm 0 0 0,clip,width=1\linewidth]{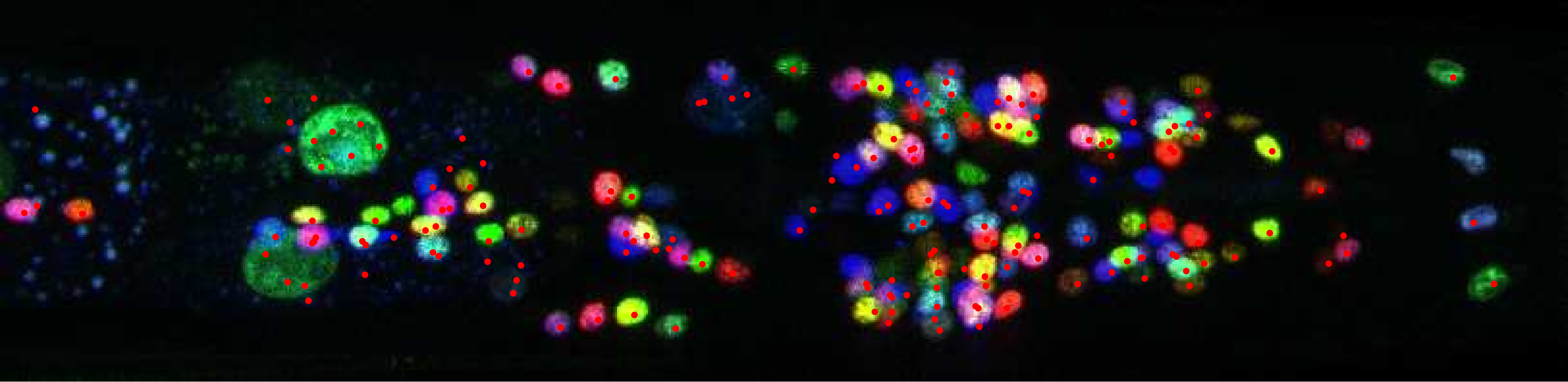}
    \caption{2D projection of 3D fluorescence microscopy image of \textit{C.elegans} head in Yemini et al. dataset. Superimposed annotation points denote neuron locations. Outliers are detections that do not correspond to neurons and missing data are undetected neurons.}
    \label{worm_annotations}
\end{figure*}

\subsubsection{Margin parameter ($\nu$) selection}
Both of the proofs of the noiseless and the noisy cases of proposition 1 rely on knowledge of the true regression coefficient and the noise variance in order to estimate the margin coefficient $\nu$ and output the optimal regression coefficient with high probability. However, in practice, as in many RANSAC-like robust regression settings, these parameters cannot be known apriori, and $\nu$ is typically determined via empirical heuristics and or cross-validation \cite{fischler1981random}.

In the noiseless case, an appropriate heuristic is choosing $\nu$ arbitrarily small since the correct regression should yield zero residual. However, for the noisy case, if available, supervised data should be used with known correspondences to estimate the actual dispersion of point correspondences.
\section{Numerical results}
To verify the theoretical guarantees of the proposed algorithms, simulated data in 3 dimensions was generated in both noisy and noiseless regimes. Furthermore, iterative solutions of $\bBeta$ and $\bP$ were obtained to demonstrate the suboptimality of local minima found using block coordinate descent for this non-convex problem.

The neuroscience application of rRWOC was demonstrated in the context of point set matching of neurons of \textit{C. elegans} worms recorded using fluorescence microscopy imaging. The matching accuracy with respect to ground truth was assessed for rRWOC as well as a robust variant of the iterative closest point (ICP) algorithm~\cite{besl1992method} known as trimmed ICP~\cite{chetverikov2002trimmed}.

\textbf{Computational setup and code:} 
All experiments were performed on an Intel i5-7500 CPU at 3.40GHz with 32GB RAM. MATLAB code for 3D versions of algorithm~\ref{alg:ranpac} are included in supplementary material along with sample \textit{C. elegans} neuron point clouds.
\begin{figure*}[!htb]
    \centering
    \includegraphics[width=1\linewidth]{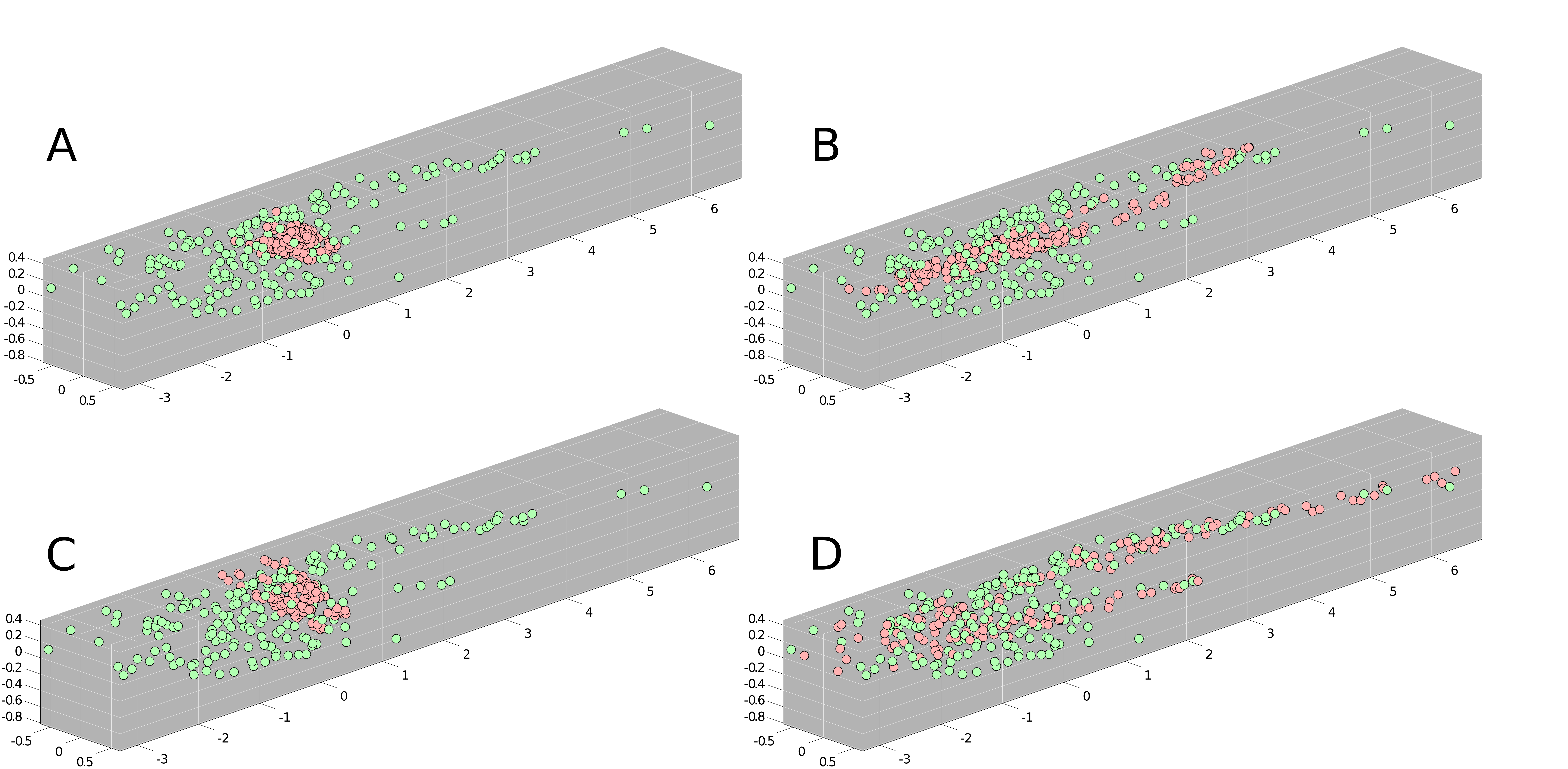}
    \caption{\textbf{A:} Unaligned point sets of reference \textit{C. elegans} neuron positions (red) and target neuron positions (green) \textbf{B:} Alignment with coherent point drift algorithm~\cite{myronenko2010point}, \textbf{C:} Alignment with iterative closest point algorithm~\cite{chetverikov2002trimmed}, \textbf{D:} Alignment with proposed algorithm~\ref{alg:ranpac}.}
    \label{worm_matches}
\end{figure*}

\subsection{Simulated data}\label{sim_results}
Three dimensional source point set $\X$ was generated by sampling $\x_j \sim \mathcal{N}(\boldsymbol{0},\boldsymbol{I}_3)$ for $j=1,\ldots,J$ where $J\in[20,\ldots,40]$. A random transformation $\bBeta$ was obtained by computing the QR factorization of a $3\times 3$ random gaussian matrix $\mathbf{M}$ such that $\mathbf{Q}\mathbf{R} =\mathbf{M}$, taking the orthonormal rotation component $\mathbf{Q}$. This was randomly scaled by a factor between $s = [0.5,1.5]$ so that $\bBeta = s\mathbf{Q}$. For $k \in [1,\ldots,19]$,~~$20-k$ inlier target points were generated by transforming a random $20-k$ subset of $\X$ by $\bBeta$ and adding gaussian noise with varying $\sigma^2$: $\Y_{\mathcal{I}} = \X_{\pi(\mathcal{I})} \bBeta + \epsilon$. Furthermore, $k$ points in $\Y$ were randomly uniformly sampled from the convex hull of the $20-k$ inlier points: $\Y_{\mathcal{O}}\sim \text{U}[\mathcal{C}(\Y_{\mathcal{I}})]$.
This procedure yielded two unordered multisets, $\X\in\R^{J \times 3}$ and $\Y\in\R^{20 \times 3}$.
Using these unordered multisets as input to rRWOC, the regression coefficients $\hat{\bBeta}$ were estimated. If $\|\hat{\bBeta} - \bBeta\|_F \leq 1e-3$, the event was considered a correct recovery, otherwise a failure. The margin parameter $\nu$ was set to be $\nu = \sigma$. Also, using the randomized algorithm~\ref{alg:ranpac}, the success probability parameter was set to $\delta = 0.9$.

This procedure was repeated 100 times for varying $k=1,\ldots,19$, varying $\sigma^2$ and varying $J=20,\dots,40$ to assess the empirical recovery rate as a function of outlier amount, SNR and missing correspondences in the target, respectively.
The recovery rates vs. outlier ratio, and SNR can be seen in figure~\ref{simulated_resultsl}-middle. The recovery rates vs. missing data ratio and SNR can be seen in figure~\ref{simulated_resultsl}-left. Lastly, the comparison of the recovery rate of exhaustive and randomized rRWOC versus iterative closest point can be seen in figure~\ref{simulated_resultsl}-right.

These empirical results demonstrate that for a sufficiently high SNR and outlier ratio of less than $50\%$, the proposed algorithm yields almost perfect recovery rates. Furthermore, the comparisons with iterative closest point algorithm (ICP) show that rRWOC is much more robust to outliers than ICP since the inclusion of any outliers results in failure of ICP to recover the true transformation.

\subsection{Neuron matching of \textit{C. elegans}}\label{c_elegans}

For this application, we have used the publicly available \textit{C. elegans} fluorescence imaging dataset of Nguyen et al. \cite{nguyen2017automatically} found at \url{http://dx.doi.org/10.21227/H2901H} as well as the neuronal position dataset provided in \cite{yemini2019neuropal}. The worm \textit{C. elegans} is a widely known model organism for studying the nervous system due to the known structural connectome of the 302 neurons it contains. The data provided 3D z-stack images of the head of 14 worms that each consists of approximately 185 to 200 neurons captured under confocal microscopy using florescent tagged protein GFP. In figure~\ref{worm_annotations}, the depth-colored 2D projection of an image frame can be seen superimposed with annotation points delineating the locations of neurons. Figure~\ref{worm_annotations} also highlights the need for a method of matching and aligning worm point clouds that is robust to outliers or missing associations. Here, we define outliers as points where there is no neuron present and define missing data as neurons with no detection present.

Of the 14 datasets of the head neurons of \textit{C.elegans} worms, random pairs were drawn to be the source and target point sets. From the remaining worms, the positional covariance of each neuron was estimated using the supervised alignment method of \cite{evangelidis2018joint}. Since the positional variance of each neuron was uniquely identified using training data, we used variable margin parameters for rRWOC such that $\nu_l = \underset{i=1,2,3}{\max} \lambda_i(\Sigma_l)$ where $\Sigma_l$ is the covariance matrix of the $l$th neuron and $\lambda_i(\cdot)$ denotes the $i$th eigenvalue. 
Randomized RWOC (algorithm~\ref{alg:ranpac}) was deployed with $\delta = 0.9$. The results were compared with iterative closest point(ICP)~\cite{besl1992method} as well as coherent point drift (CPD)~\cite{myronenko2010point} algorithms.

The recovery rates in terms of recovering the transformation $\bBeta^*$ as well as the permutation $\bP^*$, are summarized in table~\ref{table_1}. In general, rRWOC was able to recover both the transformation and permutation better than ICP and CPD, which tend to be initialization-dependent. In all of the experiments, ICP and CPD were initialized with random rotation. rRWOC is invariant to initialization since it is not a descent based method.

\begin{table*}[!tb]
    \centering
    \begin{tabular}{|c|c|c|c|c|c|c|c|c|c|}
    \hline
        \textbf{Method} & $TP$ & $FP$ & $TN$ & $FN$ & $ACC$ & $F1$ & $PREC$ & $REC$ & $MD$\\
        \hline
         \hline
         rRWOC   & 155 & 41 & 0 & 40 & 0.65 & 0.79 & 0.79& 0.79& 2.83 \\
         \hline
          ICP~\cite{chetverikov2002trimmed}& 10&186&0&185&0.02&0.055&0.05&0.05&3.4\\
         \hline 
          CPD~\cite{myronenko2010point}&20&176&0&175&0.05&0.10&0.10&0.10&10.89 \\
         \hline
    \end{tabular}
    \caption{Transformation recovery and permutation recovery by rRWOC, ICP, and CPD in the \textit{C.elegans} dataset. TP = true positive, FP = false positive, TN = true negative, FN = false negative, ACC = accuracy, F1 = F1 score , PREC = precision, REC = recall, MD = mean distance}
    \label{table_1}
\end{table*}
\subsection{Discussion}
In this paper, we expanded on the linear regression without correspondence model~\cite{unnikrishnan2018unlabeled,abid2017linear,hsu2017linear,pananjady2016linear} to account for missing data and outliers. Furthermore, we provided several exact and approximate algorithms for the recovery of regression coefficients under noiseless and noisy regimes. The proposed algorithms are combinatorial at worst with variable dimension. However, randomization procedures make the average case complexity in constant dimension tractable given enough tolerance for failure. We provided several theoretical guarantees for exact recovery and running time complexity. Furthermore, we empirically demonstrated the recovery rates of the proposed algorithms in simulated and biological data. A future algorithmic direction is to employ branch and bound techniques found in~\cite{tsakiris2019homomorphic} to reduce computational complexity of the brute force nature of the algorithms. 

The proposed methods can be thought of as a general framework for dissociating the outliers from a model-based data transformation process. The same principles can apply for the cases where either the generative noise is non-Gaussian, or some prior information exists about the structure of the outliers. Case-specific noise analysis is required for a particular model selection. Future work can focus on finding theoretical bounds on the robustness of the inlier recovery as a function of the number of outliers and the statistics of the generative and outlier distributions. 

\subsection*{Acknowledgements}
The authors would like to acknowledge the sources of funding: NSF NeuroNex Award DBI-1707398, and The Gatsby Charitable Foundation.

\appendix
\section{Proof of proposition 1}
\subsection{Noiseless case} Among the $mn$ hypothetical regression coefficients obtained through all possible pairs of $x_i$ and $y_j$, if a correct correspondence is encountered (i.e. $j=\pi^*(i))$, we have $y_{\pi(i)} = x_i \beta^*$ where $\beta^*$ is the true coefficient. Therefore if we let $\beta^i = \frac{y_{\pi(i)}}{x_i}$ then $\beta^i = \beta^*$. Using this estimate, the distances of the remaining covariates regressed to their corresponding responses is
\begin{align}
    x_l \beta^i - y_{\pi(l)}= x_l \beta^* - \y_{\pi(l)} = 0\nonumber
\end{align}
Therefore, when computing $\underset{\bP\in\mathcal{P}}\min\|\x\beta^i - \bP\y\|_2^2$ via the Hungarian algorithm~\cite{kuhn1955hungarian}, each column of the distance matrix $[D]_{p,q} = |x_p\beta^{i} - y_q|$ corresponding to inlier points in $\y$ (i.e. $q \in \mathcal{I}^*$) will have at least one zero element. Thus, the optimal assignment $\bP^i$ will include all of the permutations $\pi^i(l)=l$ since they incur zero cost. Since there are $m-k$ of them by assumption 1, then $\sum_l \boldsymbol{1}(|x_l - y_{\pi^i(l)}|\leq \epsilon/2)) \geq m-k$. This is inequality because there might be additional outlier points that are by chance close to the regressed points.

Conversely, for a pair $(x_i,y_{\pi(k)})$ where $k\neq i$, we have the estimated coefficient $\beta^{i,k} = \frac{y_{\pi(k)}}{x_i} = \frac{x_k \beta^*}{x_i}$. The distances of the remaining covariates regressed with this estimate to their corresponding responses are
\begin{align}
    x_l \beta^{i,k} - y_{\pi(l)}= \frac{x_l x_k \beta^*}{x_i} - y_{\pi(l)} = x_l\beta^*\bigg(\frac{x_k}{x_i} -1\bigg)\nonumber
\end{align}
Therefore, without loss of generality, assuming $x_l \neq 0$ (if $x_l=0$ the correspondence $(x_l,y_{\pi(l)})$ can be automatically inferred by choosing any $y_{\pi(l)}=0$. If there aren't any $y_j =0$, then this implies $x_l$ is a point without correspondence in $\y$), we have
\begin{align}
    |x_l \beta^{i,k} - y_{\pi(l)}| \geq \epsilon\nonumber
\end{align}
for some $\epsilon>0$.  $\epsilon$ can be explicitly stated as
\begin{align}
    \epsilon = \underset{i,l,k,~~ i\neq k}{\min} x_l\beta^*\bigg(\frac{x_k}{x_i} -1\bigg) \nonumber
    \end{align} 

On the other hand,
\begin{align}
    x_i \beta^{i,k} - y_{\pi(k)} = 0\nonumber
\end{align}
by construction.

Therefore, when computing $\underset{\bP\in\mathcal{P}}\min\|\x\beta^{i,k} - \bP\y\|_2^2$ via Hungarian algorithm, there will less than $m-k$ assignments in the optimal assignment $\bP^{i,k}$ such that
$|x_l - y_{\pi^{i,k}(l)}| \leq \epsilon/2$. Otherwise, this would imply the coefficient $\beta^{i,k}$ is a coefficient that explains the inliers, which by assumption 1 cannot be the case. Thus, $\sum_l \boldsymbol{1}(|x_l - y_{\pi^{i,k}(l)}|\leq \epsilon/2)) < m-k$.

This shows that the maximal cardinality of a hypothetical inlier set is at least $m-k$, and it is only achieved for a coefficient that is obtained by a correct correspondence pair. This is sufficient to show that algorithm 1 recovers the true coefficient $B^*$ under the noiseless regime.

\subsection{Noisy case}
Let the noise model of the inlier regression be $\epsilon \sim \mathcal{N}(0,\sigma^2)$. Therefore, if a correct correspondence is encountered, we have $y_{\pi(i)} = x_i \beta^* + \epsilon$ where $\beta^*$ is the true coefficient. The coefficient estimated from this pairing is $\beta^{i} = \frac{y_\pi(i)}{x_i} = \beta^* + \frac{\epsilon}{x_i}$. When this coefficient is applied to $\x$ we see that

\begin{align}
&\E(x_l \beta^i - y_{\pi(l)}) = 0,~\Var(x_l \beta^i - y_{\pi(l)}) =\bigg(\frac{x_l^2}{x_i^2}+1\bigg)\sigma^2 \nonumber \\
&\E(x_l \beta^i - y_{\pi(k)}) = (x_l-x_k)\beta^*\nonumber\\
&\Var(x_l \beta^i - y_{\pi(k)}) = \bigg(\frac{x_l^2}{x_i^2}+1\bigg)\sigma^2 \nonumber
\end{align}

Therefore, if $\sigma^2$ is small (i.e. in the SNR regime of \cite{pananjady2016linear}), we have $|x_l \beta^i - y_{\pi(l)}| < |x_l \beta^i - y_{\pi(k)}|$ for $l\neq k$ with high probability. Thus the row-wise minimal cost assignment in the Hungarian algorithm will be $\pi^i(l) = l$ with high probability. However, even if $\pi^i(l) \neq l$, if we set margin $\nu$ such that $\nu = \frac{1}{2}\underset{l,k~l\neq k}{\min} | (x_l - x_k)\beta^*|$, with high probability we will have that 
\begin{align}
   \sum_l \boldsymbol{1}(|x_l \beta^i - y_{\pi^i(l)}| \leq \nu ) \geq  \sum_l \boldsymbol{1}(|x_l \beta^{i,k} - y_{\pi^i(l)}| \leq \nu )
\end{align}
where $\beta^{i,k}$ denotes the regression coefficient obtained via incorrect correspondence $\beta^{i,k} = \frac{y_{pi(k)}}{x_i}$.
Therefore, if $\sigma^2$ is sufficiently small, with high probability, algorithm 1 recovers the coeffient $\beta^i = \beta^* + \frac{\epsilon}{x_i}$ for some $i \in \mathcal{I}$ where $\mathcal{I}$ denotes the set of inliers.

\bibliographystyle{unsrt}
\bibliography{refs}

\end{document}